\documentclass[final]{cvpr}
\usepackage{times}
\usepackage{hhline}
\usepackage{epsfig}
\usepackage{graphicx}
\usepackage{amsmath}
\usepackage{amssymb}
\usepackage{pifont}
\usepackage{hhline}
\usepackage{url}

\usepackage{booktabs}      
\usepackage{nicefrac}   
\usepackage{microtype}
\usepackage{graphicx}
\usepackage{csquotes}
\usepackage{subfigure}
\usepackage{wrapfig}
\usepackage[shortlabels]{enumitem}
\usepackage{multirow}
\usepackage{floatrow}
\usepackage{diagbox}
\usepackage{verbatim}
\usepackage{amsmath}
\usepackage{bbm}
\usepackage{bm}
\usepackage{comment}
\usepackage{amsthm}
\usepackage{amssymb}
\usepackage{amsmath}
\usepackage{nicefrac}
\usepackage{epigraph}
\usepackage[shortlabels]{enumitem}
\newtheorem{theorem}{Theorem}

\newtheorem{definition}[theorem]{Definition}

\newtheorem{remark}{Remark}

\numberwithin{theorem}{section}
\usepackage[pagebackref=true,breaklinks=true,colorlinks,bookmarks=false]{hyperref}

\def\cvprPaperID{3156} % *** Enter the CVPR Paper ID here
\def\confYear{CVPR 2021}

\begin{document}

%%%%%%%%% TITLE
\title{Identifying Invariant Texture Violation for Robust Deepfake Detection}

\author{Xinwei Sun\thanks{Equal contribution.}\\
Microsoft Research Asia\\
{\tt\small xinsun@microsoft.com}
% For a paper whose authors are all at the same institution,
% omit the following lines up until the closing ``}''.
% Additional authors and addresses can be added with ``\and'',
% just like the second author.
% To save space, use either the email address or home page, not both
\and
Botong Wu$^{*}$\\
Peking University\\
{\tt\small botongwu@pku.edu.com}
\and
Wei Chen\\
Microsoft Research Asia\\
{\tt\small wche@microsoft.com}
}

\maketitle

%%%%%%%%% ABSTRACT
\begin{abstract}

Existing deepfake detection methods have reported promising in-distribution results, by accessing published large-scale dataset. However, due to the non-smooth synthesis method, the fake samples in this dataset may expose obvious artifacts (\textit{e.g.}, stark visual contrast, non-smooth boundary), which were heavily relied on by most of the frame-level detection methods above. As these artifacts do not come up in real media forgeries, the above methods can suffer from a large degradation when applied to fake images that close to reality. To improve the robustness for high-realism fake data, we propose the \textbf{In}variant \textbf{Te}xture \textbf{Le}arning (InTeLe) framework, which only accesses the published dataset with low visual quality. Our method is based on the prior that the microscopic facial texture of the source face is \textbf{inevitably} violated by the texture transferred from the target person, which can hence be regarded as the invariant characterization shared among all fake images. To learn such an invariance for deepfake detection, our InTeLe introduces an auto-encoder framework with different decoders for pristine and fake images, which are further appended with a shallow classifier in order to separate out the obvious artifact-effect. Equipped with such a separation, the extracted embedding by encoder can capture the texture violation in fake images, followed by the classifier for the final pristine/fake prediction. As a theoretical guarantee, we prove the \emph{identifiability} of such an invariance texture violation, \textit{i.e.}, \emph{to be precisely inferred from observational data}. The effectiveness and utility of our method are demonstrated by promising generalization ability from low-quality images with obvious artifacts to fake images with high realism.

\end{abstract}

%%%%%%%%% BODY TEXT
\section{Introduction}

The advances of computing capability and deep generative models (\textit{e.g.}, Auto-Encoder) make it possible to synthesis fake videos or images that swap facial information from the source person by the one from the target person, which is known as \textit{Deepfakes} \cite{korshunov2018deepfakes}, such as \textit{identity swap} and \textit{expression swap} \cite{rossler2019faceforensics++, tolosana2020deepfakes}. By leveraging deep convolutional neural networks, current Deepfake methods can manipulate in a microscopic level that is hard to be discerned by human detectors. The propagation of these fake images on the Internet can cause potential harm in aspects like cybersecurity, defamation of reputation, and fake news circulation. An example is the recent fake news that "Aids is over" announced by Donald Trump \footnote{Although it was originally  generated to attract the public's attention on the AIDS issue, this video can incur misleading on social media.}. To counter-mitigate these side effects, a robust \footnote{Here robustness means generalize well on high-realism fake images.} and effective detection method is desired. 

% \begin{comment}
%\begin{figure}[h!]
%\centering
%    \includegraphics[width=0.8\textwidth]{Figure/trump.png}
%\label{fig:trump}
%\caption{The deepfake image, which is about the eradication of AIDS announced by by US president Donald Trump, can incur misleading on social media.}
%\end{figure}
%\end{comment}

%For example, one can steal and resell the personal information using deepfakes \footnote{https://nypost.com/2019/09/04/chinese-deepfake-app-sparks-concerns-over-identity-theft/}. 

By accessing the public large-scale dataset such as FaceForensics++ \cite{rossler2019faceforensics++} which may be with low visual quality, existing methods reported promising results on the in-distribution test data. However, without meticulous post-processing during synthesis, these low-quality fake images can expose obvious artifacts such as color mismatch and stark visual contrast, which are hence far from realism. The methods trained to capture these artifacts \cite{afchar2018mesonet, li2018exposing} may not generalize well to real forgeries that either with the better visual quality or with post-compression to smooth out the artifacts. As illustrated in Fig.~\ref{fig:deepfake}, with better synthesis method adopted in \cite{li2020celeb}, the fake images in newly published Celeb-DF dataset \cite{li2020celeb} can exhibit a large improvement of visual quality. Such an improvement cause the degradation of above detection methods \cite{li2020celeb, rossler2019faceforensics++} from FaceForensics++ to Celeb-DF (0.955 $\to$ 0.655 in terms of AUC), as summarized in Tab.~\ref{tab:stat}. 

\begin{figure*}[h!]
\centering
    \includegraphics[width=0.95\textwidth]{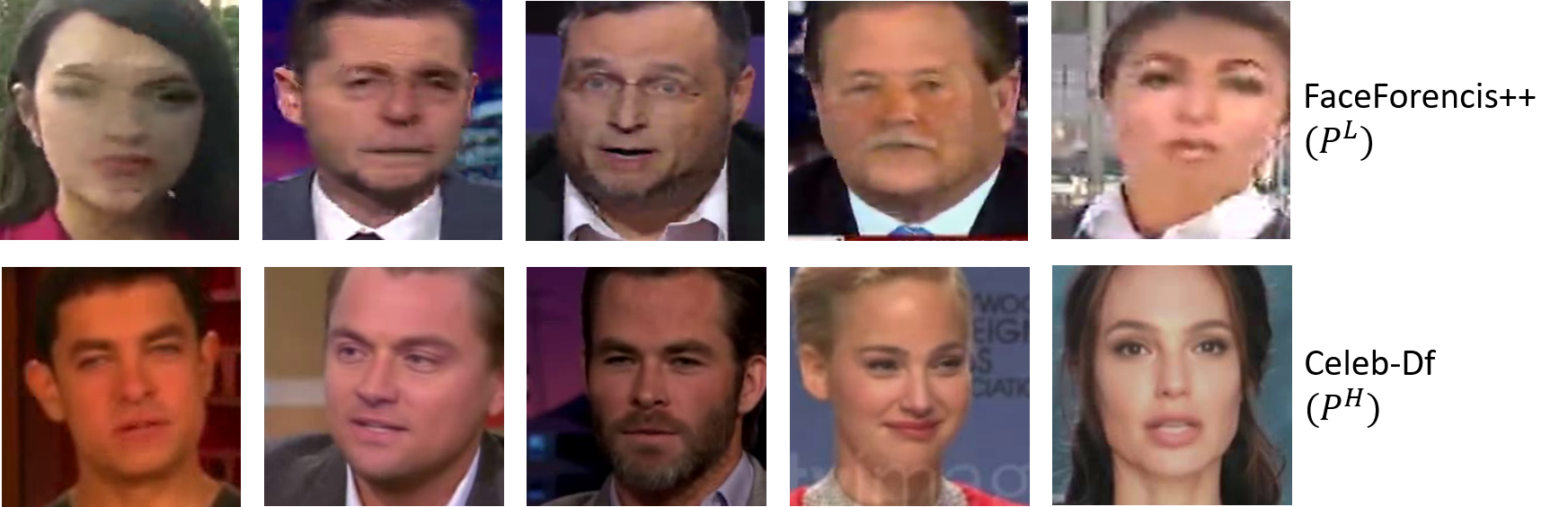}
\label{fig:deepfake}
\caption{The fake images respectively from Faceforencis++ \cite{rossler2019faceforensics++} (the 1st row) and the Celeb-DF \cite{li2020celeb} (the 2nd row). The ones in the 1st row exhibit the obvious artifacts, corresponding to the training data from $\mbox{P}^{\mathrm{L}}$ (in Def.~\ref{def:pl}) that we access. The 2nd row remove the artifacts by implementing better synthesis method ($\mbox{P}^{\mathrm{H}}$ in Def.~\ref{def:pw}), as the test arena to validate the robustness of proposed methods.
%The fake images respectively from Faceforencis++ \cite{rossler2019faceforensics++} (the 1st row) and the Celeb-DF \cite{li2020celeb} (the 2nd row) and the benchmark images \cite{rossler2019faceforensics++} (the 3rd row). The ones in the 1st row exhibit the obvious artifacts, corresponding to the training data from $\mbox{P}^{\mathrm{L}}$ (in Def.~\ref{def:pl}) that we can access. The 2nd and 3rd rows remove the artifacts by implementing careful post-processing method and compression (respectively denoting $\mbox{P}^{\mathrm{W}_a}$ and $\mbox{P}^{\mathrm{W}_b}$ in Def.~\ref{def:pw})), and can play as the test arena to validate the robustness of proposed methods.
}
\end{figure*}

Apart from artifacts that may not come up in the wild (\textit{e.g.}, those vivid and circulated fake images), we in this paper exploit an important prior that invariably hold across all fake images. As an invariance, it refers to that \emph{the facial texture (\textit{i.e.}, multitude of microscopic details and surface orientation) for one person is uniquely distributed} \cite{tanaka2016parts,liu2005recognizing}. That is, different facial regions may expose heterogeneous textures though, they give rise to the holistic regularity such as geometric coherence or high-order smoothness of spatial variation. This prior has been exploited in the literature of face recognition \cite{pierrard2007skin, hadid2006face, zhang2009local}. Replacing the facial region with the one from other sources (\textit{i.e.} people) would violate the holistic regularity mentioned above. Different from the obvious artifacts due to careless synthesis, such an \emph{invariance with texture-wise violation} corresponds to high-frequent signals that cannot be perceived by human detectors, making the detection extremely difficult.

To generalize well on real forgeries from the data with low visual quality, we propose a novel framework, namely \textbf{In}variant \textbf{Te}xture \textbf{Le}arning (InTeLe) to identify this invariant texture violation. Our InTeLe is based on a latent generative model (as shown in Fig.~\ref{fig:dag}) that encapsulates two sources of clues of fake images with low quality: the artifacts during synthesis and the texture violation. Specifically, we introduce hidden variables $Z$ to model the high-level latent components that contain the texture information. Besides, the generation process of pristine images differs from the one for fake images that can suffer from artifacts. This is illustrated by the additional arrow $Y \to X$ (where $Y$ denotes the binary pristine/fake label) illustrated in the left image in Fig.~\ref{fig:dag}. For high-quality fake images with removed artifacts, this arrow does not exist (in the right image in Fig.~\ref{fig:dag}) and the generating process of $X$ no longer depends on $Y$.

To learn the invariant texture information for generalization, our InTeLe introduces an auto-encoder structure, with the encoder extracting the texture information, followed by a classifier for pristine/fake prediction. Motivated by the $Y \to X$ in the latent generative model, our decoders are with different branches for generating fake and pristine images. To model such a generating difference reflected by the artifacts, we append a shallow classifier on decoded images to regularize the decoder, leaving the invariant texture violation solely encoded into the latent components for generalization. As a guarantee, we prove that this invariant texture information can be precisely inferred from observational distribution, \textit{i.e., identifiability}. To demonstrate the robustness of InTeLe, we test it on more realistic fake data in Celeb-DF \cite{li2020celeb} and benchmark images in \cite{rossler2019faceforensics++}. The empirical results show a large-margin improvement (\textit{e.g.} $6.5\%$ AUC on Celeb-DF) compared with conventional supervised learning. The contribution is summarized as follows: 
\begin{itemize}
    \item We are \emph{the first} to consider the texture violation, as an invariance for deepfake detection.
    \item We propose an auto-encoder-based model that separates out the artifact to identify the invariance for prediction. 
    \item We prove the identifiability of the invariant texture violation, giving rise to the observational distribution.
    \item We generalize to datasets with high realism (from low-qualtiy data) and achieve better results than others. 
\end{itemize}

%assume the generation process from $Z$ to the facial image $X$ is dependent on the binary real/fake label $Y$, with of the pristine image is different from the one of the fake ones, due to the artifacts. 

% Specifically, we construct a causal graph that encapsulates the We therefore propose that for detecting fake images with low quality, there are two clues: the artifacts and texture inconsistency, with the latter being an invariant information shared among all fake images. 

%differ from the pristine ones only in microscopic level that are hard to be detected by human observers. In particular, the \textit{DeepFakes} \cite{korshunov2018deepfakes} implemented the deep AE to generate the facial expression of the target person from the one of the source person. With the powerful propagation ability on that swap the source person's face by the one in the target person, which is known as . Since the target of the targethat are hard to be detected by humans. The widespread of these digital images can employ \textbf{AI-synthesized deepfake images -> potential harm -> effective detection method is desired (frame-level detection)}Deepfake algorithms can create fake images and videos that humans cannot distinguish them from authentic ones. widespread can be harmful shows an fake video hta  

%-----------------------------------------------------------------------------

\section{Related Work}

%Existing works can be broadly categorized into video-level prediction and the frame/image-level prediction. \newline
 
\noindent \textbf{Video Detection.} The video-based detection method can exploit temporal discrepancies \cite{sabir2019recurrent,guera2018deepfake}, and also the physiological aspects \cite{lyu2020deepfake} including incoherent head poses \cite{yang2019exposing}, abnormal eye blinking patterns \cite{li2018ictu} or behaviours \cite{agarwal2019protecting}; the exploitation of spatial and motion information \cite{wang2020video}. \textit{However}, in many real scenarios, the image-level manipulation is enough to cause potential harm, such as cybersecurity or defaming the politician's reputation by generating inappropriate expression in some contexts. Therefore, a robust frame-based detection method is more desired. 

\noindent \textbf{Image Detection.} For detecting fake images, most existing methods \cite{li2018exposing, matern2019exploiting, afchar2018mesonet,zhou2017two, zhou2018learning, liu2018image} trained on the Faceforensics++ \cite{rossler2019faceforensics++} (as one of the first large-scale public data) to learn the obvious artifacts due to imperfect synthesis (we will elaborate the details in section~\ref{sec:lgm}). These artifacts can visually cause stark contrast, which corresponds to low-frequency signals with respect to the human observers and hence far from realism. Using shallow networks, these low-frequent signals can be captured to achieve accurate prediction \cite{afchar2018mesonet}. 
%To particularly mention, the generation methods adopted in \cite{rossler2019faceforensics++} first extract the facial image, followed by manipulation (replacement or reenactment) and then warping back to the original face. The warping back procedure can induce non-smooth boundary, color mismatch between the source person and the target person in terms of skin \cite{afchar2018mesonet}, if no post-processing is implemented. These artifacts can visually cause stark contrast, which corresponds to low-frequency signals with respect to the human observers and hence far from realism. Using shallow networks, these low-frequent signals can be captured to achieve accurate prediction \cite{afchar2018mesonet}. 
Another line of methods uses a data-driven method by directly training DNNs, such as XceptionNet \cite{rossler2019faceforensics++, nguyen2019capsule, korshunov2020deepfake}. Benefited from the ability to learn high-frequency signals by deep layers, these DNN can learn both artifacts and texture violation (\textit{but} without separation of each effect). However, these methods cannot generalize to real forgeries, \textit{e.g.}, \cite{li2020celeb} with color correction and smooth mask detection. The \cite{li2020celeb} reported this performance degradation.

\noindent \textbf{Key Difference.} Compared with existing works, our method focus on identifying the texture violation characterized by microscopic details, as an invariance shared among all fake images. Specifically, \textit{we targeting on generalizing from the Faceforensics++ which expose obvious artifacts, to Celeb-DF with fake images that are close to reality.} Particularly note that, rather than simply performing well on Celeb-DF (since we do not use data from Celeb-DF), our goal is to propose a robust model that can exploit invariant texture information for prediction. This goal is commonly believed to be the future prospects for deepfakes detection \cite{tolosana2020deepfakes,lyu2020deepfake}. Compared with conventional feature-extraction methods (\textit{e.g.}, \cite{galbally2013image}), our method is able to extract the high-frequency signals equipped with deep neural network.

%------------------------------------------------------------------------
\section{Methodology}

% In the subsequent subsection, we will provide the identifiability result (\textit{i.e.}, the possibility of precise inference) associated with $p(z|y)$ (with $p(y)$ can be estimated from the training data in $\mathrm{L}$) and $f_x$. 

\noindent \textbf{Problem Setting} With accessible Deepfake dataset which can expose obvious artifacts due to unsatisfactory synthetic quality, we expect to generalize well on fake data with high realism. Formally speaking, our training data contains $N$ samples $\{x_i,y_i\} \overset{i.i.d}{\sim} \mbox{P}^{\mathrm{L}}(x,y)$ (with ``L" stands for low quality and $\mbox{P}^{\mathrm{L}}(X,Y)$ is defined in Def.~\ref{def:pl}), where $(x,y) \in \mathcal{X} \times \mathcal{Y}$ denote the image and pristine/fake label with the $\mathcal{Y} := \{0,1\}$ (The ``0" denotes the pristine while the ``1" denotes the fake). Our goal is to learn a \textbf{\emph{robust}} predictor $f: \mathcal{X} \to \mathcal{Y}$ that can \textbf{\emph{generalize well from $\mbox{P}^{\mathrm{L}}$} to real world media forgeries with high visual quality}, which that is characterized by $\mbox{P}^{\mathrm{H}} \neq \mbox{P}^{\mathrm{L}}$ (''H" stands for high quality and $\mbox{P}^{\mathrm{H}}$ is defined in Def.~\ref{def:pw}): $\min_f \mathbb{E}_{p^{\mathrm{H}}(x,y)} \mathbbm{1}(f(x) \neq y)$. 
%Besides, we denote $\tilde{\mathcal{Y}} := \{0,1,...,K\}$ as the label space of sub-types of images: the $0$ denotes the pristine image and the $K \geq k > 0$ denotes the fake image implemented with the $k$-th manipulation method. 

\noindent \textbf{Outline.} We start with a latent generative model in section~\ref{sec:lgm} to model two effects in fake images with low visual quality: the artifacts and texture violation for $\mbox{P}^{\mathrm{L}}$. This generative model motivates the our InTeLe framework in section~\ref{sec:intele} which identifies the invariant texture violation by separating out the artifact-effect. Finally, we in section~\ref{sec:iden} provide a theoretical guarantee that this invariant texture information can be precisely inferred, \textit{i.e., identifiability}.

  % Denoting $\tilde{\bm{y}} := \mathbbm{1}(\bm{y} > 0)$ for pristine/fake image classification scenario,

% ---------------------------------------------------------------------------------------------------------

\subsection{Latent Generative Model}
\label{sec:lgm}

For manipulated/fake images with low visual quality, there are two clues for detection: \textit{(i)} the intrinsic texture violation from another person; and \textit{(ii)} the synthetic artifacts such as splicing boundaries, color mismatch. We assume that the \textit{(i)} is shared by all fake images due to the following prior regarding the facial texture: \textit{The facial texture, as a generalization of a multitude of microscopic details and surface orientation, is uniquely determined for each person \cite{tanaka2016parts,liu2005recognizing}}. For the manipulation that leverages the information from another person, the distribution of the source person's facial texture will be inevitably violated. To understand the \textit{(ii)} that only come up in fake images far from realism, taking the synthesis method adopted in Faceforencis++ \cite{rossler2019faceforensics++} as an illustration: after generating the target facial region from that of the source person, the \cite{rossler2019faceforensics++} implemented a warping back operation in order to match with the original face. It is this warping back procedure that induces non-smooth boundary, the color mismatch between the source person and the target person in terms of skin \cite{afchar2018mesonet}, leaving obvious artifacts.

\begin{figure}[h!]

\centering
\begin{tabular}{cccc}
   \hspace{0.0in}    
    \includegraphics[width=0.3\textwidth]{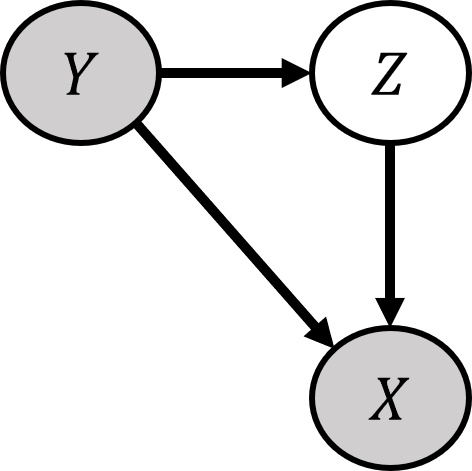} 
	& \hspace{0.3in}   \includegraphics[width=0.3\textwidth]{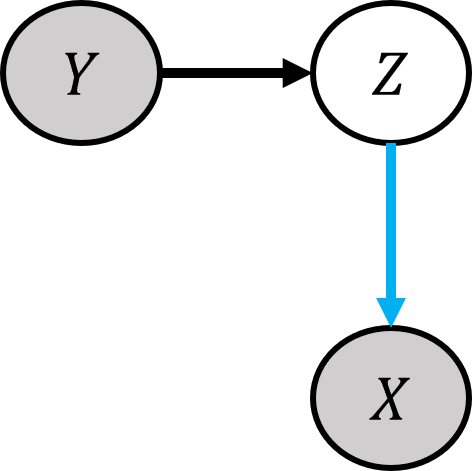} \\
	%&  \hspace{0.0in}   \includegraphics[width=0.25\textwidth]{Figure/dag_benchmark4.png} \\
	\hspace{0in}  (a) DAG for $\mbox{P}^{\mathrm{L}}$ & \hspace{0.3in}  (b) DAG for $\mbox{P}^{\mathrm{H}}$ 
	%& \hspace{0in}  (c) DAG for $\mbox{P}^{\mathrm{W}_b}$ 
\end{tabular}
\caption{Directed Acylic Graph (DAG) of $\mbox{P}^{\mathrm{L}}$ which generates low-quality fake images that expose obvious artifacts, $\mbox{P}^{\mathrm{H}}$ which generates high-quality fake images due to better synthesis process. With the artifacts removed, the generating process of $\mbox{P}^{\mathrm{H}}$ no longer depends on the pristine/fake label. The blue arrow in $\mbox{P}^{\mathrm{H}}$ means that the corresponding generating mechanism/structural equation is changed compared to that of $\mbox{P}^{\mathrm{L}}$.}
\label{fig:dag}
\end{figure}

We encapsulate these two clues into a latent generative model illustrated in the leftmost image of Fig.~\ref{fig:dag}. As shown, the hidden variable $Z \in \mathbb{R}^{d_z}$ is introduced to model the explanatory factors of facial image $X$; these explanatory factors encode the texture information of each local facial region such as five sense organs, hair, beard, \textit{etc}. The obvious artifact (\textit{i.e., (ii)}) used to differentiate the fake images is modeled by the dependency on the label $y$ during the generating process. The existing method either exploited the artifacts (\textit{a.k.a, (i)}) or both \textit{(i)} and \textit{(ii)}; and could achieve high detection recall on $\mbox{P}^{\mathrm{L}}$. However, these methods cannot generalize well on high-realism fake images from $\mbox{P}^{\mathrm{H}}$ which removes the artifact via better synthesis method or compression and only expose microscopic texture violation which is not easy to be detected by human detectors. Due to such a removal of the artifact, we assume that the generating process of $\mbox{P}^{\mathrm{H}}$ only depends on the latent variable, as commonly assumed in the literature of latent generative model \cite{kingma2014auto,suter2019robustly}. This is shown by the missing link from $Y \to X$ in the middle image of Fig.~\ref{fig:dag}, and the $\mbox{P}^{\mathrm{H}}(x|z) = \mbox{P}^{\mathrm{L}}(x|z,y=0)$ (recall that the $y=0$ denotes the pristine image). The formal definition of $\mbox{P}^{\mathrm{L}},\mbox{P}^{\mathrm{H}}$ are introduced in Def.~\ref{def:pl},~\ref{def:pw}.

% in which the fake images are with better visual quality, due to better synthesize quality; and the ones from  which implement post-compression to smooth out the artifacts. 

%The formal definition for the generative model of the , together with $\mbox{P}^{\mathrm{L}}$ are given as follows:
%\begin{comment}
%\begin{definition}[Gen]
%\label{def:pl}
%The structural causal model \cite{pearl2009causality} of $\mbox{P}^{\mathrm{L}}$, which is named \textbf{C}ausal \textbf{T}exture \textbf{G}enerative \textbf{M}odel (CTGM) is defined by (i) a Directed Acylic Graph $G = (\mathcal{V},\mathcal{E})$ with $\mathcal{V} := (X,Y,Z)$ such that $X \in \mathcal{X}$, $Y \in \mathcal{Y}$ and $Z \in \mathcal{Z}$ respectively denotes the input, output and latent variables; and $\mathcal{E} := \{Y \to Z, Y \to X, Z \to X$ in which each arrow $a \to b$ denoting the generating process from $a$ to $b$; and (ii) generating mechanisms $\mathcal{F}:=\{y \gets f_y(\varepsilon_y), z \gets f_z(y,\varepsilon_z), X \gets f_x(y,z,\varepsilon_x)$ that together with independent exogenous variables $\{\varepsilon_y,\varepsilon_x,\varepsilon_z\}$ assign the distributions $p_{f_y}(y), p_{f_z}(z|y), p_{f_x}(x|z,y)$ for $\mathcal{V}$. 
%\end{definition}
%\end{comment}
%According to Causal Markov Condition \cite{pearl2009causality}, since the exogenous variables $\varepsilon_x, \varepsilon_$
\begin{definition}[Latent Generative Model for $\mbox{P}^{\mathrm{L}}$]
\label{def:pl}
The generative model \cite{pearl2009causality} of $\mbox{P}^{\mathrm{L}}$ is defined by (i) a Directed Acylic Graph (DAG) as illustrated in the leftmost image in Fig.~\ref{fig:dag} and denoted as $G = (\mathcal{V},\mathcal{E})$ with $\mathcal{V} := (X,Y,Z)$ containing the input $X$, output $Y$ and latent variables $Z$ and $\mathcal{E} := \{Y \to Z, Y \to X, Z \to X\}$ characterizing the generating direction; and (ii) generating distributions $\left( p_{f_y}(y), p_{f_z}(z|y), p_{f_x}(x|z,y) \right)$ for each element in $\mathcal{V}$. 
\end{definition}

Note that this generative model differs from the commonly adopted graphical model with $Z \to X$ in \cite{kingma2014semi} in the additional arrow $Y \to X$, implying the difference in the generating process for pristine images and fake images. Specifically, this difference lies in the obvious artifacts that do not come up for natural pristine images, which is due to the non-smooth warping back operation. With improved synthesis method or compression to remove such artifacts, this difference no longer exists, making the microscopic texture manipulation the sole clue for detection. 
%In fact, the $\left( p^{\mathrm{L}}_{f_y}(y), p^{\mathrm{L}}_{f_z}(z|y), p^{\mathrm{L}}_{f_x}(x|z,y) \right)$ are induced by generating functions $\mathcal{F}:=\{y \gets f_y(\varepsilon_y), z \gets f_z(y,\varepsilon_z), x \gets f_x(y,z,\varepsilon_x)\}$ and exogenous variables $\{\varepsilon_y,\varepsilon_x,\varepsilon_z\}$. 

%Further, if such exogenous variables are independent, we have according to Causal Markov Condition \cite{pearl2009causality} that, 
%\begin{equation}
%\label{eq:dag}
%    p^{\mathrm{L}}(X,Z,Y) = p_{f_y}(y)p_{f_z}(z|y)p_{f_x}(x|z,y).
%\end{equation}
%Similarly, we define the generative models for $\mbox{P}^{\mathrm{W}_a}$ and $\mbox{P}^{\mathrm{W}_b}$ in the following: 
\begin{definition}[Latent Generative Model for $\mbox{P}^{\mathrm{H}}$] 
\label{def:pw}
The generative model for $\mbox{P}^{\mathrm{H}}$ is the same to Def.~\ref{def:pl} except the missing arrow in $Y \to X$ and $p_{f^{\mathrm{H}}_x}(x|z) = p_{f_x}(x|z,y=0)$.
\end{definition}

According to Def.~\ref{def:pl},~\ref{def:pw}, the $p(z,y)$ hence the mechanism of $p(y|z)$, as the predicting mechanism with only texture inconsistency, is shared among all generative models, therefore is defined as the \emph{invariance with texture violation}:
\begin{definition}[Invariance with Texture Violation] 
\label{def:invariance}
We define the predicting mechanism $p(y|z)$ as the \textbf{invariance with texture violation}, which is shared between $\mbox{P}^{\mathrm{L}}$ in Def.~\ref{def:pl} and $\mbox{P}^{\mathrm{H}}$ in Def.~\ref{def:pw}.
\end{definition}

By accessing only data from $\mbox{P}^{\mathrm{L}}$ which mixes the texture violation with the artifact-effect, our goal is to identify the \emph{invariance with texture violation} that is also shared by $\mbox{P}^{\mathrm{L}}$, by separating out the artifact-effect. Such a learned invariance is used for final prediction on $\mbox{P}^{\mathrm{H}}$. To learn the invariance in Def.~\ref{def:invariance}, note that the generating mechanism of image $X$ in $\mbox{P}^{\mathrm{H}}$, although is different from, however is similar to the one in $\mbox{P}^{\mathrm{L}}$: $X \gets f_x(Z,Y,\varepsilon_X)$ for $\mbox{P}^{\mathrm{L}}$ v.s $X \gets f_x(Z,Y=0,\varepsilon_X)$ for $\mbox{P}^{\mathrm{H}}$, providing us the opportunity to identify $f_x$, with which we can infer the latent components $Z$ from $X$ and finally predict using $p(y|z)$. In the subsequent section, we provide our InTeLe framework which learns the $p(y|z)$ from $\mbox{P}^{\mathrm{L}}$ and generalize to the data from $\mbox{P}^{\mathrm{H}}$.

\begin{figure*}[h!]
\centering
    \includegraphics[width=0.95\textwidth]{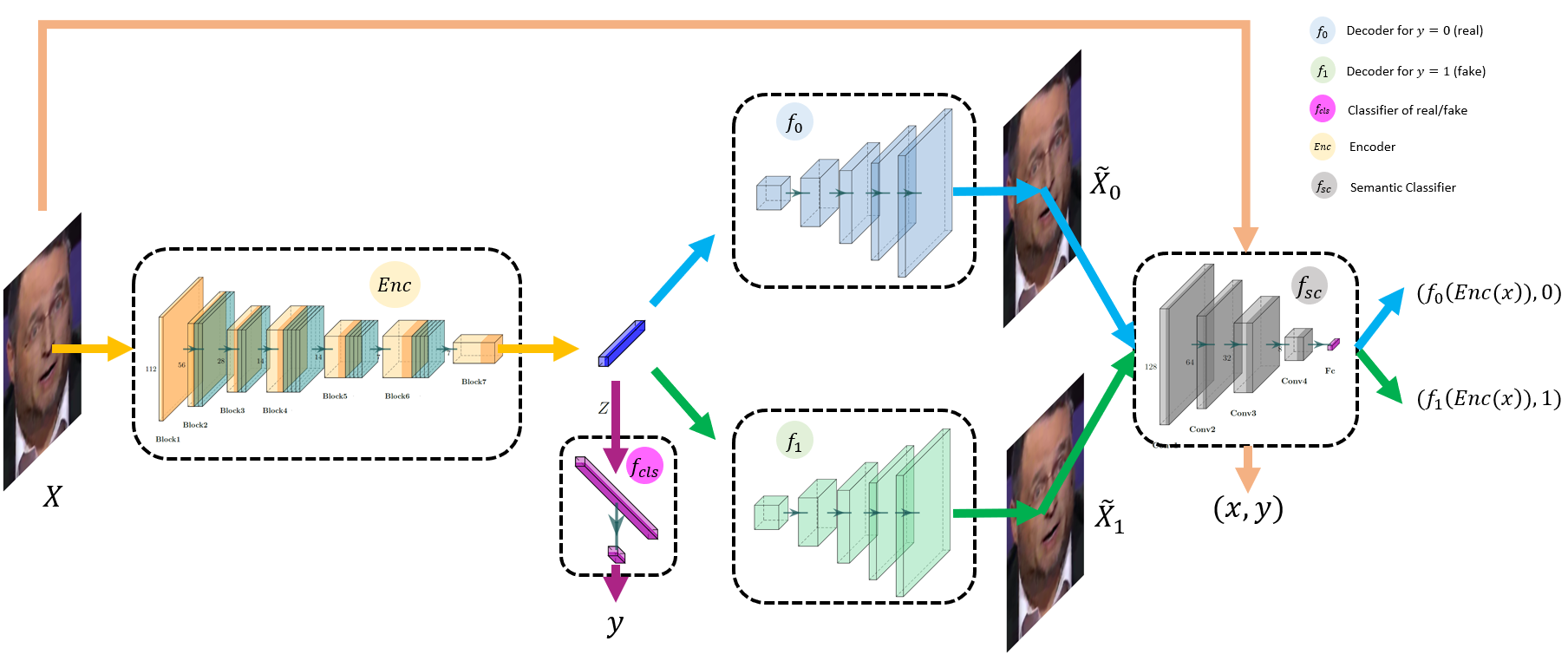}
\caption{Our InTeLe framework that learns from $\mbox{P}^{\mathrm{L}}$ and aims at learning the \emph{invariance with texture violation}, \textit{i.e.}, $p(y|z)$ to generalize on $\mbox{P}^{\mathrm{H}}$. The InTeLe is composed of three modules: (i) auto-encoder with encoder $\mathrm{Enc}$ to extract the texture information and two-branch decoder (\textit{i.e.}, $f_0, f_1$) corresponding to generating process of pristine/fake images; (ii) the invariant classifier $f_{\mathrm{cls}}$ which exploits the information of the texture inconsistency; (iii) the semantic classification $f_{\mathrm{SC}}$ to encode the artifact-effect, by classifying $f_0(\mathrm{Enc}(x))$ as 0 and $f_1(\mathrm{Enc}(x))$ as 1; besides, it trains on $(x,y)$ to regularize $f_0,f_1$ not to learn trivial solutions. }
\label{fig:framework}
\end{figure*}

%-------------------------------------------------------------------------------------------------------
\subsection{Invariant Texture Learning Framework}
\label{sec:intele}

We introduce our learning method, namely \textbf{In}variant \textbf{Te}xture \textbf{Le}arning, with the goal of identifying the invariance with texture violation (\textit{i.e.}, $p(z|y)$) from data in $\mbox{P}^{\mathrm{L}}$ and generalize on data from $\mbox{P}^{\mathrm{H}}$, by separating it from the artifact-effect. The whole pipeline is illustrated in Fig.~\ref{fig:framework}. As shown, our InTeLe is based on the Auto-Encoder (AE), as a deterministic version ($\varepsilon_x = 0$, which can be proved to satisfy the identifiability in section~\ref{sec:iden}) of Variational AE to learn the generative model. To model the artifacts and texture violation in $\mbox{P}^{\mathrm{L}}$, our \textbf{(i)} encoder is expected to extract the latent components $Z$ that incorporates the texture information of local facial regions; and our \textbf{(ii)} two-branch decoders respectively correspond to the generating process of fake and pristine images, which echos the dependency of $Y$ during generation of $X$ in Fig.~\ref{fig:dag}. To further enforce the decoder to capture the obvious artifact-effect, we append a shallow network (\textit{i.e.}, $f_{\mathrm{SC}}$) to capture the low-frequency information (which is easy to be observe by humans). By separating out this artifact, the extracted latent features $Z$ only encode the texture consistency for predicting pristine/fake. The modules of auto-encoder ($\mathrm{Enc},f_0,f_1$), classifier $f_{\mathrm{cls}}$ and the semantic classifier $f_{\mathrm{SC}}$ participant in the end-to-end training. We explain each module in more details.

\noindent  \textbf{Auto-Encoder Module.} As early mentioned, the information of texture inconsistency of each local region is embedded into the latent components $Z$ of the whole image $X$, and the artifacts are the ones that left during the synthesizing of fake images. Motivated by these \textit{priories}, we leverage the auto-encoder framework to model such two effects. The latent components extracted by the encoder is expected to capture the invariant texture inconsistency. For the decoder, to model the artifacts due to synthesis, we propose a two-branch version with one (denoted as $f_0$) for generating the pristine images and the other (denoted as $f_1$) for generating the fake images. Since texture information corresponds to high-frequency signals that can be captured by deep neural networks, we choose the network structure for encoder with deep layers (\textit{e.g.}, XceptionNet, EfficientNet), and the details will be introduced in the experimental part. The overall loss function for encoder $\mathrm{Enc}$ and the decoder $f_0,f_1$ is:
\begin{align}
    \label{eq:gen}
       \mathcal{L}_{\mathrm{AE}}(\mathrm{Enc},f_0,f_1) & = \sum_{i=1}^n \mathbbm{1}(y_i =0) * \Vert X_i - f_0(\mathrm{Enc}(x_i)) \Vert_F^2 \nonumber \\
    + &  \mathbbm{1}(y_i =1) * \Vert X_i - f_1(\mathrm{Enc}(x_i)) \Vert_F^2 
\end{align}

\noindent \textbf{Semantic Classifier Module.} To further ensure that the artifacts can be captured by the difference between $f_0$ and $f_1$, we append a shallow semantic classifier ($f_{\mathrm{SC}}$ in Fig.~\ref{fig:framework}) which has been shown to only capture the obvious low-frequency signals (\textit{e.g.}, the semantic pattern in object classification). We expect that for any image $x$, this $f_{\mathrm{SC}}$ can only exploit artifacts as a difference between $f_0(\mathrm{Enc}(x))$ and $f_1(\mathrm{Enc}(x))$ for prediction. Specifically, for each image either is pristine or fake, the $f_{\mathrm{SC}}$ is trained to classify the estimated images from $f_0$ (which expose artifacts) as 0 and the one from $f_1$ (no artifacts exposed) as 1:
\begin{align}
\label{eq:lsc-ce}
     \mathcal{L}_{\mathrm{CE}}(\mathrm{Enc},f_0,f_1,f_{\mathrm{SC}}) & =  -\sum_{i=1}^n \log{ \left(1 - f_{\mathrm{SC}}(f_0(\mathrm{Enc}(x_i))) \right) } \nonumber \\
     &  + \log{\left(f_{\mathrm{SC}}(f_1(\mathrm{Enc}(x_i)))\right)}.
\end{align}
Further, to avoid learning the trivial solution for $f_0$ and $f_1$ (such as making the first pixel of the reconstructed image different), we add an auxiliary branch of pristine/fake classification with image $x$ directly as input to avoid learning $f_0,f_1$ that are trivial for separation by $f_{\mathrm{SC}}$:
\begin{align}
    \label{eq:lsc-aux}
    \mathcal{L}_{\mathrm{aux}}(f_{\mathrm{SC}}) & = -\sum_{i=1}^n \mathbbm{1}(y_i = 0) * \log{ \left(1 - f_{\mathrm{SC}}(x_i) \right) } \nonumber \\
     & + \mathbbm{1}(y_i = 1) * \log{\left(f_{\mathrm{SC}}(x_i)\right)}.
\end{align}
Combining with the Eq.~\eqref{eq:lsc-ce}, the loss function is:
\begin{align}
    \mathcal{L}_{\mathrm{SC}}(f_0,f_1,f_{\mathrm{SC}}) = \mathcal{L}_{\mathrm{CE}}(f_0,f_1,f_{\mathrm{SC}}) + \alpha * \mathcal{L}_{\mathrm{aux}}(f_{\mathrm{SC}}), \nonumber 
\end{align}
with $\alpha \geq 0$ denoting the hyper-parameter to represent the extent of regularization.

\noindent  \textbf{Classifier Module.} With the effect of artifacts separated by the semantic classifier $f_{\mathrm{SC}}$ following after $f_0, f_1$, the $Z$ only contains the texture information, which is therefore taken as input for the final classifier that is trained by:
\begin{align}
    \label{eq:cls}
   \mathcal{L}_{\mathrm{cls}}(f_{\mathrm{cls}}, \mathrm{Enc}) & = -\sum_{i=1}^{n} \mathbbm{1}(y_i = 0) * \log{ \left(1 - f_{\mathrm{cls}}(\mathrm{Enc}(x_i)) \right) } \nonumber  \\
     & + \mathbbm{1}(y_i = 1) * \log{\left(f_{\mathrm{cls}}(\mathrm{Enc}(x_i))\right)}.
\end{align}

\noindent \textbf{Training and Testing.} In summary, the overall loss functions is the combination of the losses in the modules mentioned above: the reconstructed loss (\textit{i.e.}, Eq.~\eqref{eq:gen}) followed by the shallow semantic classifier to separate the effect of artifacts from the texture information, which is for final classification. The overall loss $\mathcal{L}$ for training is: 
\begin{align}
    \label{eq:overall-loss}
    & \mathcal{L}(f_0,f_1,\mathrm{Enc}, f_{\mathrm{SC}}, f_{\mathrm{cls}}) = \mathcal{L}_{\mathrm{AE}}(\mathrm{Enc},f_0,f_1) \nonumber \\
    & + \lambda_1 *  \mathcal{L}_{\mathrm{SC}}(f_0,f_1,f_{\mathrm{SC}}) + \lambda_2 *  \mathcal{L}_{\mathrm{cls}}(f_{\mathrm{cls}}, \mathrm{Enc}), 
\end{align}
which trains all modules in an end-to-end scheme. The $\lambda_1,\lambda_2$ are hyper-parameters that balance the effects of each module. During testing, given a new sample $x_{\mathrm{new}}$, we feed the extracted texture information $Z$ (\textit{i.e.}, $\mathrm{Enc}(x_{\mathrm{new}})$) into the classifier $f_{\mathrm{cls}}$ for pristine/fake prediction. 

%----------------------------------------------------------------------------------------------------------

\subsection{Identifiability of the Texture Violation}
\label{sec:iden}

In this section, we provide a theoretical guarantee for our InTeLe to ensure that the \emph{invariance with texture violation} (\textit{i.e.}, $p(y|z)$) can be identified, \textit{i.e.} precisely inference from the observational distribution (here refers to $p^{\mathrm{L}}(x,y)$). Our analysis is inspired by the recent analysis in nonlinear ICA \cite{khemakhem2019variational,khemakhem2020ice} with adaptation on our generative model in Fig.~\ref{fig:dag} with additional dependency of $Y$ during generation. Besides, we here generalize the definition of $Y$ to $\mathcal{Y} := \{0,...,m\}$ which denotes the space of sub-types of images: the $0$ denotes the pristine image and the $k \leq m$ denotes the fake image implemented with the $k$-th manipulation method. Our goal is to identify the generating mechanism $f_x$ for inference of latent variable $Z$ (as the intrinsic texture information for training) and hence $p(y|z)$ for prediction. In the following, we assume that the latent variables $Z|Y$ is generated from exponential family, \textit{i.e.}, 
\begin{align}
%\label{eq:p-exp}
p^{\mathrm{L}}_{f_z}(z|y) = \prod_{i=1}^{d_z} \exp\Big( \sum_{j=1}^{k_z} T_{i,j}(z_i) \Gamma_{y,i,j} + B_i(z_i) -  A_{y,i} \Big), \nonumber 
\end{align}
where $f_z$ is associated with the sufficient statistics $\{T_{i,j}(z_i)\}$; the $\{\Gamma_{y,i,j}\}$ denotes the natural parameters; and $\{B_i\}, \{A_{y,i} \}$ denote the base measures and normalizing constants to ensure the integral of distribution equals to 1. Let $\mathbf{T} \!:=\! \left[\mathbf{T}_{1},...,\mathbf{T}_{d_z}  \right]$ $\!\in\! \mathbb{R}^{k_z \times d_z}$ $\big(\mathbf{T}_{i} \!:=\! [T_{i,1},...,T_{i,k_z}], \forall i \in [d_z]\big)$ and $\bm{\Gamma}_y \!:=\! \left[\bm{\Gamma}_{y,1},...,\bm{\Gamma}_{y,d_z} \right]$ $\!\in\! \mathbb{R}^{k_z \times d_z}$ $\big(\bm{\Gamma}_{y,i} \!:=\! [\Gamma_{y,i,1},...,\Gamma_{y,i,k_z}], \forall i \in [d_z]\big)$. Our goal is to identify the parameters $\theta:\{\mathbf{T},f_x\}$ that give rise to the observational distribution $p^{\mathrm{L}}(x,y)$. Since $Z$ is assumed to model the high-level abstractions/concepts in the latent generative model \cite{kingma2014auto}, we assume that $d_z < d_x$ and that $\varepsilon_x \in \mathbb{R}^{d_x - d_z}$ such that $f_x$ is bijective. The following theorem ensures that the $\theta$ is identifiable under a linear and pointwise transformation, which is similar to the result in \cite{khemakhem2019variational}. 

\begin{theorem}[Identifiability]
\label{thm:iden}
Consider the DAG associated with generative model which has the following factorization, 
\begin{equation}
\label{eq:dag}
    p^{\mathrm{L}}(x,z,y) = p_{f_y}(y)p_{f_z}(z|y)p_{f_x}(x|z,y).
\end{equation}
under the following assumptions:
\begin{enumerate}[(a),topsep=0pt,noitemsep]
    \item The sufficient statistics $\{T_{i,j}\}$ are differentiable and with nonzero derivatives almost everywhere.
    \item There exists $m:=k_zd_z+1$ values of $y$, \textit{i.e.}, $y_0,y_1,..,y_{m}$ such that the matrix $[ \bm{\Gamma}_{y_1} - \bm{\Gamma}_{y_0}$ $,...,\bm{\Gamma}_{y_{m}} - \bm{\Gamma}_{y_0} ]$ is invertible
\end{enumerate}
we have that if $\theta$ and $\tilde{\theta} := \{\tilde{\mathbf{T}},\tilde{f}_x\}$ give rise to the same observational distribution, \textit{i.e.}, $p_\theta^{\mathrm{L}}(x,y) = p_{\tilde{\theta}}^{\mathrm{L}}(x,y)$ for any $(x,y) \in \mathcal{X} \times \mathcal{Y}$ distribution, then there exists an invertible matrix $A \in \mathbb{R}^{k_zd_z \times k_z d_z}$ and vector $b \in \mathbb{R}^{k_z d_z}$ such that $\bm{T}([f_x^{-1}]_{\mathcal{Z}}) = A \tilde{\bm{T}}([\tilde{f}_x^{-1}]_{\mathcal{Z}}) + b$ \footnote{Here the $[f]_{\mathcal{A}}$ denotes the function $f$ restricted on the dimension index that belongs to index set $\mathcal{A}$}.
\end{theorem}

\begin{proof}

Suppose $\theta, \tilde{\theta}$ give the same observational distribution that $p_{\theta}(x,y) = p_{\tilde{\theta}}(x,y)$, then we have:
\begin{align}
    \int p_{f_y}(y)p_{f_z}(z|y)p_{f_x}(x|z,y)dz = \int p_{\tilde{f}_y}(y)p_{\tilde{f}_z}(z|y)p_{\tilde{f}_x}(x|z,y)dz,
\end{align}
since $p_{f_y}(y) = p_{\tilde{f}_y}(y)$, then we have that 
\begin{align}
    \int p_{T^z,\Gamma}(z|y)p_{f_x}(x|z,y)dz = \int p_{\tilde{T}^z,\tilde{\Gamma}}(z|y)p_{\tilde{f}_x}(x|z,y)dz. 
\end{align}
According to the rule of change of variables that $\bar{x}:=f_x(z,y)$, we have 
\begin{align}
& \int p_{T^z,\Gamma}([f^{-1}(\bar{x})]_{\mathcal{Z}}|y)p_{\varepsilon_x}(x-\bar{x}) |J_{f^{-1}}(\bar{x})|d\bar{x} \nonumber \\
= & \int p_{\tilde{T}^z,\tilde{\Gamma}}([\tilde{f}^{-1}(\bar{x})]_{\mathcal{Z}}|y)p_{\varepsilon_x}(x-\bar{x}) |J_{\tilde{f}^{-1}}(\bar{x})|d\bar{x}. 
\end{align}
This implies that $(p_{T,\Gamma,f,y} * p_{\varepsilon_x})(x) = (p_{\tilde{T},\tilde{\Gamma},\tilde{f},y} * p_{\varepsilon_x})(x)$. Taking the Fourier transformation on both sides, we have that $F(p_{T,\Gamma,f,y})(\omega) \phi_{\varepsilon_x}(\omega) = F(p_{\tilde{T},\tilde{\Gamma},\tilde{f},y})(\omega) \phi_{\varepsilon_x}(\omega)$. This implies that $F(p_{T,\Gamma,f,y})(\omega) \phi_{\varepsilon_x}(\omega) = F(p_{\tilde{T},\tilde{\Gamma},\tilde{f},y})(\omega)$, from which we have $p_{T,\Gamma,f,y} = p_{\tilde{T},\tilde{\Gamma},\tilde{f},y}$ by taking the inverse Fourier transformation on both sides. Here, the $p_{T,\Gamma,f,y}$ is 
\begin{align}
    p_{T,\Gamma,f,y}(x) = p_{T,\Gamma}(f^{-1}(x)|y)|J_{f^{-1}}(x)|.
\end{align}
By taking the logarithmic on both sizes of $p_{T,\Gamma,f,y} = p_{\tilde{T},\tilde{\Gamma},\tilde{f},y}$, we have that 
\begin{align}
    \langle T, \Gamma(y) \rangle = \langle \tilde{T}, \tilde{\Gamma}(y) \rangle + b(y) + g(x),
\end{align}
where $b(y)$ is related to $y$ and $g(x)$ is the function of $x$. Define $\bar{\Gamma}(y) := \Gamma(y) - \Gamma(y_1)$, we have that 
\begin{align}
    L^\top T = \tilde{L}^{\top} \tilde{T}  + b. 
\end{align}
According to the condition $b)$ and applying the result in \cite{khemakhem2019variational}, we can have that $A:=(L^{\top})^{-1}\tilde{L}$ is invertible. The proof is completed.
\end{proof}

\begin{remark}
The condition (b) puts a requirement on the number and the diversity of manipulating classes to identify the invariant mechanisms $\theta$. Notwithstanding that this condition may not be satisfied in real scenarios (since the accessible training data may only contain a limited number of manipulation types), we empirically find that incorporating $y$ (into decoder) can help identify the texture information, even if with $Y$ being as binary pristine/fake label. 
\end{remark}
This theorem states that under a deterministic setting, the latent variable can be identified up to linear transformation (with an invertible matrix) and a point-wise transformation. Besides, we believe that latent components (the texture and various abstractions regarding each facial region) are deterministic to the facial image. Therefore, we propose to learn such invariant texture information using an auto-encoder rather than a variational one.

%\subsection{Learning Method}

%\subsection{Transferring to Real Media Forgeries} 

%In this section, we consider transferring the learned mechanisms to fake images \textit{in the wild}, such as those circulated on the Internet. Specifically, we consider the test images generated from the domain $\mathrm{W}_a$ and $\mathrm{W}_b$. 

%\subsubsection{Domain Adaptation}

\section{Experiments}

\begin{table*}[h!]
\centering
\caption{The statistics of pristine/fake videos/images and the performance of XceptionNet \cite{rossler2019faceforensics++} trained on C40 dataset of Faceforensics++. Compared to the Faceforensics++, the Xception suffers from performance degradation on both Celeb-DF with better visual quality and benchmark images with compression to smooth out the artifacts. }
\setlength{\tabcolsep}{2pt}
\small
\begin{tabular}{c|c|c|c|c|c|c|c|c|c|c|c}
\toprule
\diagbox{Attributes}{Dataset} & \# Pristine Videos & \# Fake Videos & \# Pristine images & \# Fake images & AUC \cite{rossler2019faceforensics++} & Precision \cite{rossler2019faceforensics++} \\
\hline
Faceforensics++ \cite{rossler2019faceforensics++}  & 1,000 & 1,000 & 509.9k & 1,830.1k & 0.955 & - \\
\hline
Celeb-DF \cite{li2020celeb} & 590 & 5,639 & 225.4k &  2,116.8k & 0.655 &  0.810 \\   
\hline 
Benchmark in \cite{rossler2019faceforensics++} & 1,000 & - & 500 & 500 & - & 0.701 \\
\bottomrule
\end{tabular}
\label{tab:stat}
\end{table*}

In this section, we use Faceforensics++ \cite{rossler2019faceforensics++} (which expose obvious artifacts for classification \cite{tolosana2020deepfakes, afchar2018mesonet, li2020celeb}) for training, and test our model on two dataset of  ($\mbox{P}^{\mathrm{H}}$) \textit{i)} the test set of Celeb-DF \cite{li2020celeb}, and \textit{ii)} additional 1000 benchmark images in \cite{rossler2019faceforensics++}. We will introduce these three datasets in details, with statistics for each dataset summarized in Tab.~\ref{tab:stat}.

\subsection{Dataset}
\label{sec:data}

\noindent \textbf{Faceforensics++ \cite{rossler2019faceforensics++} for $\mbox{P}^{\mathrm{L}}$.} The Faceforensics++ transferred the face (or expression) from the target video to each frame of the source video to generate a fake image. Different levels of image resolution are considered: the Raw (original) with the highest resolution, the light compression C23 (constant rate quantization parameter equal to 23), and C40. To synthesize fake images, the Faceforensics++ first extracts only facial region from the source image for generating the target facial region, followed by final warping back steps such as affine transformation and shape refinement to match the non-facial regions in the source image. Such a warping procedure can exhibit visual artifacts, which can be easily captured to detect the fake images \cite{tolosana2020deepfakes, afchar2018mesonet,li2018exposing}.

%using four different approaches to generate fake videos: \textbf{(i)} the FaceSwap \footnote{https://github.com/deepfakes/faceswap} and DeepFakes \footnote{https://github.com/MarekKowalski/FaceSwap/} which belong to face replacement; \textbf{(ii)} Face2Face \cite{thies2016face2face} and NeuralTextures \cite{thies2019deferred} belonging to expression reenactment. 

\noindent \textbf{Celeb-DF \cite{li2020celeb}.} Compared to the Faceforensics++, the Celeb-DF significantly improve the image quality, which matches with the characteristics in $\mbox{P}^{\mathrm{H}}$ and provides a more challenging arena for detection methods. Specifically, the \cite{li2020celeb} adopted more wide and deep network structures of Auto-encoder to improve the resolution. Besides, it generated a facial mask that along with the facial surrounding context, can make the boundary more smooth than the convex hull adopted in Faceforensics++. It additionally corrected the color mismatch. We can therefore approximately assume that the generation of the fake image is the same as that of the natural pristine image. Although the artifacts were removed, the fake images exhibit inconsistent textures between facial and nonfacial regions for detection.

% According to previous surveys \cite{tolosana2020deepfakes}, the Celeb-DF suffers from large degradation in terms of classification accuracy. Particularly, the model trained on the Faceforensics++ can not generalize well on the Celeb-DF. 

\noindent \textbf{Benchmark images in \cite{rossler2019faceforensics++}.} In addition to Faceforensics++, the \cite{rossler2019faceforensics++} provided a competitive benchmark on 1,000 videos post-processed by unknown compression approaches to mimic the ones in real media scenarios with low resolution. This benchmark data only contains 1,000 images, each of which was manually selected as the most challenging one among all frames in the corresponding video. The dataset can be downloaded from the host-server of the \cite{rossler2019faceforensics++} with unknown ground-truth labels. As shown in Tab.~\ref{tab:stat}, the Xception trained on C40 suffers from precision degradation on benchmark images. 

%One can test the pristine/fake precision of the method on pristine, four manipulated and all images by submitting the binary labels. 

\subsection{Implementations}

We respectively adopt the Area Under the ROC Curve (AUC) and prediction accuracy as the evaluation metric for Celeb-DF and the benchmark images.

\noindent \textbf{Compared Baselines.} For Celeb-DF, we compare with \newline
\textbf{a)} The Convolutional Neural Network (CNN) supervised by cross entropy (CE). For future reference, we name it as CE ($X \to Y$) regardless of the choice of backbone. \newline
\textbf{b)} The spoof cues framework (LGSC) \cite{feng2020learning} also employed the auto-encoder framework. The LGSC was formulated as an anomaly detection method, which regularized the decoder to generate spoof cues for manipulated images. Specifically, it is composed of a spoof cue generator and an auxiliary classifier, with the spoof cue generator parameterized by U-Net to capture the multi-scale information. To expose spoof cues for manipulated images, it implemented a sparse regression loss on the distance between the residual feature map of pristine samples and the zero feature map; while for fake images, it put no constraints. \newline
\textbf{c)} To validate the effectiveness of semantic classifier, we additionally conduct another version of our method but without $f_{\mathrm{SC}}$ and the other modules are kept the same.

\noindent \textbf{Implementation Details.} For all methods, we first crop the image to only enclose the human face using the OpenFace toolkit \cite{amos2016openface}. This pre-processing has been shown \cite{rossler2019faceforensics++} to improve the detecting performance. We only keep the fake images with replacement in the training set since the Celeb-DF only contains face replacement. We implement U-Net \cite{ronneberger2015u} for the decoder, which can capture multi-scale especially fine-scale information in image reconstruction or object segmentation \cite{baheti2020eff}. We implement various backbones for the encoder (also the layers before CE), including EfficientNet-B5 \cite{tan2019efficientnet} with student-noise initialization \cite{xie2020self} and XceptionNet \cite{chollet2017xception} to validate that the effectiveness of our method can generalize to any backbone as long as it is deep enough the capture the high-frequency signals. We implement the stochastic gradient descent (SGD) as the optimizer, with the learning rate set as 0.02. The weight-decay coefficient is set to 1e-4. The batch-size is respectively set to 16 and 20 for EfficientNet-B5 and XceptionNet. The $\lambda_1,\lambda_2$ in Eq.~\eqref{eq:overall-loss} are set to 2 and 0.25, respectively. The $\alpha$ for $\mathcal{L}_{f_{\mathrm{SC}}}$ in Eq.~\eqref{eq:lsc-aux} is set to 4.

\begin{figure}[h!]
    \centering
    \includegraphics[width=0.495\textwidth]{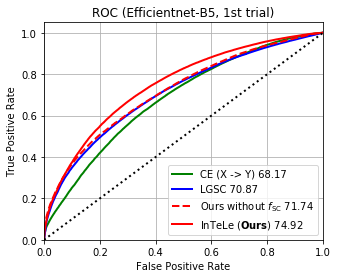}
    \includegraphics[width=0.495\textwidth]{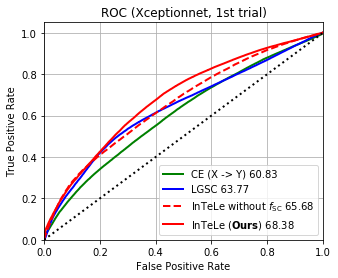}
     \includegraphics[width=0.495\textwidth]{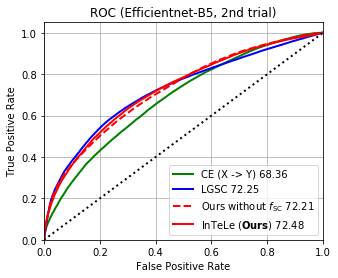}
    \includegraphics[width=0.495\textwidth]{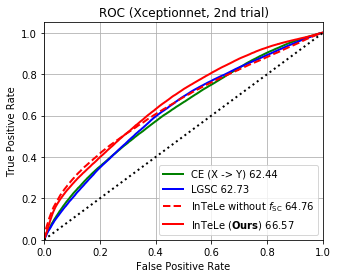}
     \includegraphics[width=0.495\textwidth]{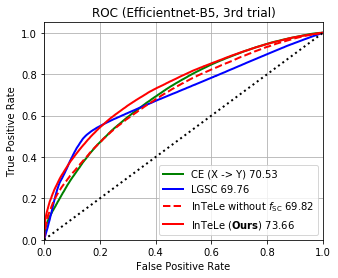}
    \includegraphics[width=0.495\textwidth]{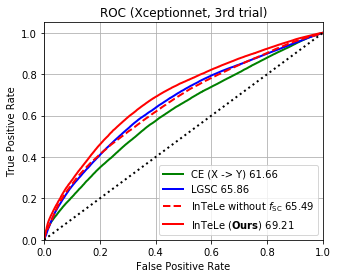}
    \caption{ROC curves of our methods and compared baselines. The left image is with EfficientNet-B5 as backbone; the right image is with the XceptionNet as backbone.}
    \label{fig:roc}
\end{figure}

\subsection{Results and Analysis}

\begin{figure*}[h!]
    \centering
    \includegraphics[width=0.918\textwidth]{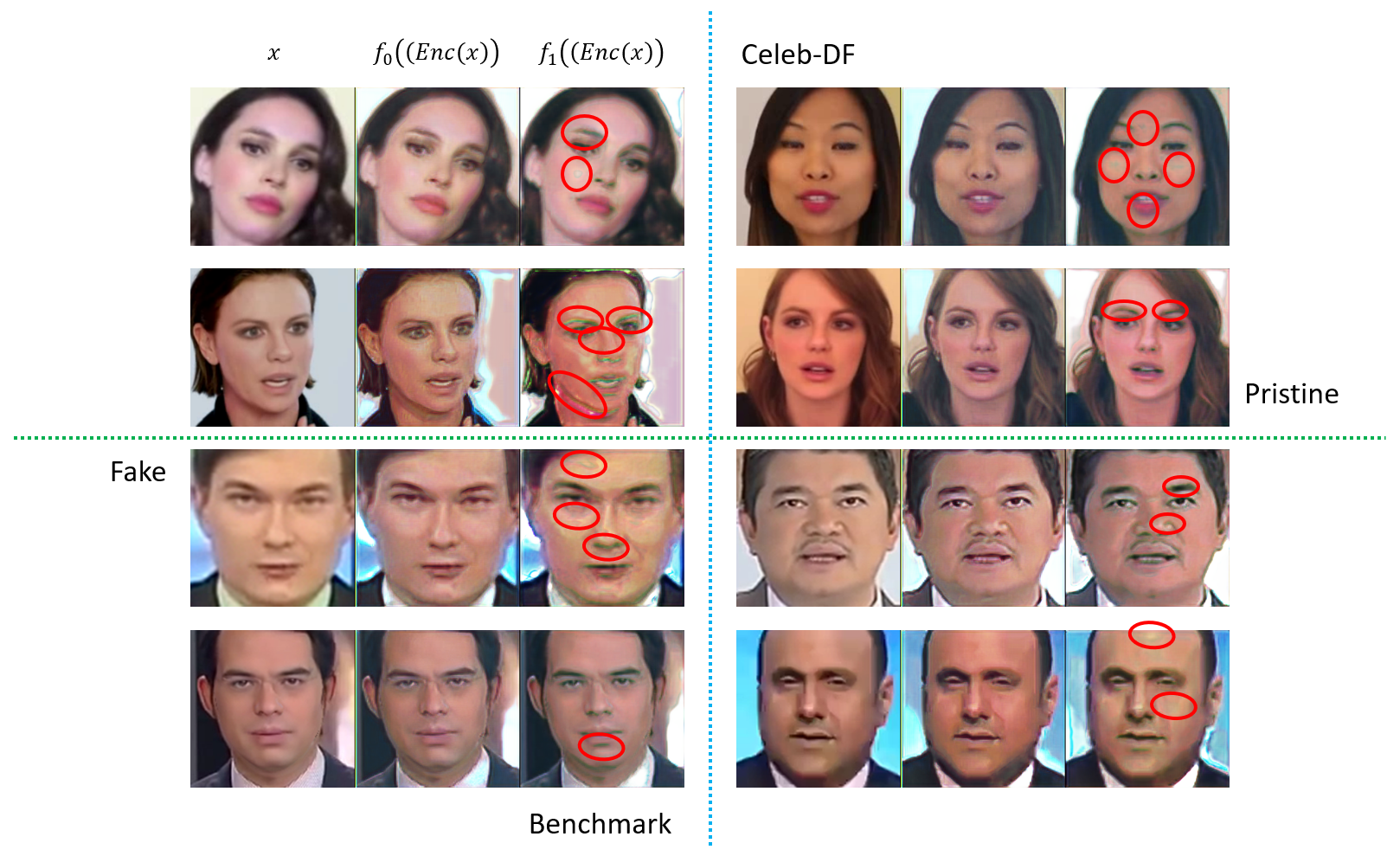}
    \caption{Separation of artifacts by the two-branch decoder in our InTeLe. The top left, top right, bottom left and the bottom right quadrants correspond to pristine images in Celeb-DF, fake images in Celeb-DF, pristine images in benchmark and the fake images in benchmark. From the left to right in each row corresponds to original image $x$, reconstructed image by $f_0$ (\textit{i.e.}, $f_0(\mathrm{Enc}(x)))$) and reconstructed image by $f_1$ (\textit{i.e.}, $f_1(\mathrm{Enc}(x))$). As marked by the red circle, the $f_1(\mathrm{Enc}(x)))$ can expose obvious artifacts whether $x$ is pristine/fake.}
    \label{fig:artifact}
\end{figure*}

We first apply our method on the test set of Celeb-DF. The average over 3 runs and standard variance for each method are summarized in Tab.~\ref{tab:auc}. The ROC curve for these experiments is also plotted in Fig.~\ref{fig:roc}. As shown, our method performs better or comparable than other methods under any backbones. Specifically, due to the modeling of abnormal patterns with triplet loss, the LGSC can also outperform the CE ($X \to Y$). Equipped with the semantic classifier $f_{\mathrm{SC}}$, our label-dependent decoder can separate out the artifact-effect, making the encoded embedding $Z$ capture cleaner \emph{invariance with texture information} (\textit{i.e.}, in Def.~\ref{def:invariance}). The effectiveness of $f_{\mathrm{SC}}$ is validated by the improvement of our InTeLe over one without $f_{\mathrm{SC}}$.

%\begin{table}[h!]
%  \centering
%  \setlength{\tabcolsep}{1.4pt}
%  \small
%  \caption{}
%    \begin{tabular}{c|c|c|c|c}
%    \hline
%    \diagbox{Backbone}{Method} & CE & LGSC & Ours (No $f_{\mathrm{sc}}$) & \textbf{Ours}  \\
%    \hline
%    EfficientNet-B5 & $69.8 \pm 1.0$ & $71.0 \pm 1.8$ & $72.6 \pm 1.6$ & $\bm{73.8 \pm 1.1}$ \\
%    \hline
%    XceptionNet  & $61.6 \pm 0.8$ &  & $65.3 \pm 0.5$ & $\bm{68.1 \pm 1.4}$  \\
%    \hline
%    ResNet-18 &  &  &  \\
%    \hline
%    \end{tabular}
%  \label{tab:auc}
%\end{table}

\begin{table}[h!]
  \centering
  \setlength{\tabcolsep}{1pt}
  \caption{AUC (mean $\pm$ std) on Celeb-DF \cite{li2020celeb} over three runs; and ACC on Benchmark. As shown, our method consistently outperform others. The ROC are shown in Fig.~\ref{fig:roc}.}
  \small
    \begin{tabular}{c|c|c|c}
    \toprule
    \multirow{2}{*}{\diagbox{Method}{Dataset}} & \multicolumn{2}{c|}{Celeb-DF \cite{li2020celeb}} & Benchmark \cite{rossler2019faceforensics++} \\
    \cmidrule{2-4}
    & EfficientNet-B5 & XceptionNet & EfficientNet-B5 \\
    \hline
    CE ($X \to Y$) & $69.8 \pm 1.0$ & $61.6 \pm 0.8$ & $83.0\%$ \\
    \hline
    LGSC \cite{feng2020learning}  & $71.0 \pm 1.8$ & $64.1 \pm 1.3$ & $84.8\%$ \cite{feng2020learning}  \\
    \hhline{=|=|=|=}
    InTeLe without $f_{\mathrm{sc}}$ & $71.3 \pm 1.0$ & $65.3 \pm 0.5$ & $84.4\%$ \\
    \hline
    InTeLe (\textbf{Ours}) & $\bm{73.8 \pm 1.1}$ & $\bm{68.1 \pm 1.4}$ & $\bm{86.4\%}$ \\
    \bottomrule
    \end{tabular}
  \label{tab:auc}
\end{table}

\noindent \textbf{Results on benchmark images in \cite{rossler2019faceforensics++}.} For benchmark data, we adopt a down-sample strategy for fake images in Faceforensics++ to make the ratio of pristine-fake approximately being 1:1 (1:2 in the original dataset). Besides, we augmented the dataset by implementing compression on the original image. Our method can achieve accuracy $\mathbf{86.4\%}$, which outperform LGSC ($84.8\%$) and is the best among the methods with published code in the leaderboard. The comparisons with the CE $X \to Y$ and InTeLe without $f_{\mathrm{SC}}$ show similar phenomena to those of Celeb-DF, as shown in Tab.~\ref{tab:auc}.

\subsection{Effect of Separating Artifacts}

To validate our InTele's effect of separating the artifacts, we visualize the triplet $(x,f_0(\mathrm{Enc}(x)),f_1(\mathrm{Enc}(x)))$ in Fig.~\ref{fig:artifact}. We pick two examples for each case: the fake of Celeb-DF; the pristine of Celeb-DF; the fake of benchmark and the pristine of benchmark, corresponding to the top-left, top-right, bottom-left and bottom-right quadrants. As marked by the red-circle, the $f_1(\mathrm{Enc}(x))$ can exhibit obvious artifacts, such as stark contrast, color mismatch whereas $f_0(\mathrm{Enc}(x))$ do not even the $x$ is fake. This differentiation validates that our $f_0,f_1$ can successfully capture the artifacts effect and \textit{hence} leave the invariant texture violation encoded in the $Z$, which can be contributed to the label-dependent decoder and appended semantic classifier $f_{\mathrm{SC}}$.

\section{Conclusions \& Discussions}

We propose the InTeLe framework which exploits the texture violation due to manipulation for Deepfakes detection. Based on the auto-encoder framework, we propose two-branch decoders appended with a semantic classifier to separate the texture information from the effect of the artifact. As an invariance shared among all fake images, the capture of this texture violation yielded a robust detection method, as validated by much better generalization results than existing artifact-based methods. 

For limitations, our identifiability theorem requires that the types of generating process (\textit{i.e.}, manipulation method) are diverse enough, which may not be satisfied in real scenarios. Besides, our method should be able to generalize to more types of manipulations that beyond the scope of face replacement and expression reenactment in this work. Such a more broadly generalization, together with the identification of the invariant texture information with a limited number of fake textures, is left in the future work.

\newpage
{\small
\bibliographystyle{ieee_fullname}
\bibliography{ref}
}

\newpage

%\section{Supplementary Information}

%\subsection{Proof of Theorem~\ref{thm:iden}}

\end{document}